\documentclass[conference]{IEEEtran}
\IEEEoverridecommandlockouts
\usepackage{cite}
\usepackage[hidelinks]{hyperref}
\usepackage{amsmath,amssymb,amsfonts}
\usepackage{amsthm}
\theoremstyle{plain}

\newtheorem{lemma}{Lemma}
\newtheorem{corollary}{Corollary}
\usepackage{algorithmic}
\usepackage{graphicx}
\usepackage{textcomp}
\usepackage{xcolor}
\usepackage[table]{xcolor}
\def\BibTeX{{\rm B\kern-.05em{\sc i\kern-.025em b}\kern-.08em
    T\kern-.1667em\lower.7ex\hbox{E}\kern-.125emX}}
\usepackage[ruled,vlined,linesnumbered]{algorithm2e}
\usepackage{array}
\usepackage{tikz}
\usetikzlibrary{arrows.meta,positioning,calc,decorations.pathmorphing}
\usepackage[caption=false,font=normalsize,labelfont=sf,textfont=sf]{subfig}
\usepackage{textcomp}
\usepackage{stfloats}
\usepackage{url}
\usepackage{verbatim}
\usepackage{enumitem}
\usepackage{multirow}
\usepackage{booktabs}
\usepackage{fancyhdr}
\fancypagestyle{firstpage}{
    \fancyhf{} % Clear all header and footer fields
    \fancyhead[L]{To be published as a conference paper at ICDE 2027}
    
}

\begin{document}

\title{HCFM: Helmholtz Component Flow Matching for Time Series Anomaly Detection}

\author{\IEEEauthorblockN{Shyam Sundar Saravanan}
\IEEEauthorblockA{\textit{Department of Computer Science} \\
\textit{Missouri University of Science and Technology}\\
Rolla, MO, USA\\
ssdmw@mst.edu}
\and
\IEEEauthorblockN{Tie Luo}
\IEEEauthorblockA{\textit{Department of ECE} \\
\textit{University of Kentucky}\\
Lexington, KY, USA \\
t.luo@uky.edu}
\and
\IEEEauthorblockN{Sanjay Madria}
\IEEEauthorblockA{\textit{Department of Computer Science} \\
\textit{Missouri University of Science and Technology}\\
Rolla, MO, USA\\
madrias@mst.edu}
\and
}

\maketitle
\thispagestyle{firstpage}

\begin{abstract}
Flow Matching (FM) is an efficient simulation-free approach for training flow-based generative models, making it a promising foundation for scalable time-series anomaly detection (TSAD). However, standard FM learns a monolithic vector field in which nominal transport and local density-changing behavior are entangled. This creates a central challenge for TSAD; local expansion and contraction carry crucial anomaly-sensitive information, yet isolating this signal requires expensive Jacobian-trace estimation over the entire field. We propose \emph{Helmholtz Component Flow Matching} (HCFM), an unsupervised TSAD framework that parameterizes the vector field prediction network with separate transport, compression, and residual heads. The transport head is constrained to be divergence-free, so divergence estimation is restricted to the compression head, while a lightweight residual head absorbs the remaining mismatch. This structured parameterization improves scalability and detection performance by separating volume-preserving nominal transport from density-changing compression, yielding cleaner anomaly signals from distinct deviations of nominal behavior. Experiments on large-scale univariate and multivariate TSAD benchmarks show that HCFM consistently improves detection accuracy while maintaining efficient inference compared with 16 baselines. The code is available at \url{https://github.com/TrustAIoT/HCFM__time-series-anomaly/}

\end{abstract}


\section{Introduction}

Time-series anomaly detection (TSAD) is commonly formulated as a one-class learning problem: a model is trained on nominal data and later scores future windows according to their deviation from learned normal behavior. This setting arises in industrial monitoring, cyber-physical systems, server telemetry, medical signals, spacecraft data, and other domains where anomalies are rare, costly, and difficult to label. Many deep TSAD methods therefore model nominality through reconstruction error, prediction error, likelihood, or latent-space distance~\cite{hundman2018detecting, xu2022anomaly, su2019robust, tuli2022tranad, zamanzadeh2024deep}.

A key challenge is that anomalies are not always expressed as large point-level error. Abnormal windows may instead appear through localized structural changes: suppressed variation, abnormal amplitude expansion, distorted morphology, or regime shifts that alter how nearby samples are distributed~\cite{cook2019anomaly,blazquez2021review}. Such effects suggest that useful anomaly evidence may lie not only in how far a window is from nominal behavior, but also in how the learned nominal model locally deforms around it. However, most generative TSAD methods collapse these effects into a single reconstruction error, likelihood, or distance score, making it difficult to isolate the specific source of abnormality~\cite{zamanzadeh2024deep}.

Flow Matching (FM) provides an appealing foundation for scalable TSAD. FM learns a vector field that transports samples from a simple base distribution to the data distribution using an efficient simulation-free objective~\cite{lipman2023flow}. This gives FM several practical advantages over common generative TSAD approaches: it avoids decoder reconstruction as the sole scoring mechanism, does not require iterative diffusion-style denoising, and does not impose the invertibility constraints of normalizing flows. These properties make FM attractive for efficient generative modeling of nominal time-series windows.

However, standard FM predicts the vector field with a monolithic network, so nominal transport and local density-changing behavior are entangled in a single output. Divergence measures local expansion or contraction of the learned flow, and can therefore reveal anomaly-sensitive deformation around abnormal windows. Yet extracting this signal from standard FM requires estimating the Jacobian trace over the entire predicted field, which is computationally expensive and does not distinguish volume-preserving movement from density-changing behavior. Recent divergence-regularized FM methods, such as Flow Divergence Matching (FDM)~\cite{huang2026improving}, recognize the importance of divergence by encouraging the learned field to better match the divergence of the target probability path. Nevertheless, FDM still applies divergence matching to a monolithic field predictor, so nominal transport and local expansion or contraction remain entangled in its output.

We propose \emph{Helmholtz Component Flow Matching} (HCFM), an unsupervised TSAD framework that parameterizes the FM field prediction with transport, compression, and residual heads. HCFM does not train separate flow models; rather, it predicts a single total vector field through an additive head structure. The transport head is constrained to be divergence-free, encouraging it to represent volume-preserving nominal movement. The compression head captures density-changing behavior through divergence, so divergence estimation is required only for this head rather than for the full predicted field. A lightweight residual head absorbs residual mismatch without dominating the transport and compression heads.

Figure~\ref{fig:hcfm_motivation} illustrates the key idea. Standard FM uses a monolithic field that entangles nominal transport with local density-changing behavior. HCFM preserves a single final field but structures it into divergence-free transport, density-changing compression, and lightweight residual components. This component-structured parameterization localizes divergence-sensitive behavior to the compression head rather than the full field, preserving FM efficiency while exposing more interpretable anomaly evidence.

We evaluate HCFM on large-scale univariate and multivariate TSAD benchmarks, including 250 UCR anomaly datasets and six multivariate datasets. Across these settings, HCFM improves over 16 representative baselines. On UCR, HCFM achieves a 10.4\% relative best-F1 improvement over the best non-HCFM baseline, while on multivariate benchmarks it obtains the best aggregate results and the best per-dataset VUS-PR and best F1 on five of six datasets. HCFM also improves inference efficiency over monolithic divergence-based Flow Matching, reducing latency by 35\% relative to FDM. These results show that component-structured Flow Matching yields more informative anomaly scores while maintaining efficient inference.

In summary, our contributions are as follows:
\begin{itemize}[leftmargin=*]
\item Identify a structural limitation of monolithic Flow Matching for TSAD: nominal transport and local density-changing behavior are entangled in the learned vector field, making divergence-based anomaly evidence difficult to isolate.

\item Introduce HCFM, a component-structured Flow Matching framework that represents one final vector field through divergence-free transport, density-changing compression, and lightweight residual components, while retaining standard FM supervision.

\item Develop component-specific divergence scoring that localizes Jacobian-trace estimation to the compression component and combines this signal with reverse-flow base energy, yielding complementary evidence for different anomaly modes.

\item Establish the effectiveness and efficiency of the proposed structure through large-scale experiments on 250 UCR datasets and six multivariate benchmarks, including controlled ablations that isolate the contribution of field structure, divergence source, and residual allocation.

\item Open source code for reproducibility available at \url{https://github.com/TrustAIoT/HCFM__time-series-anomaly/}

\end{itemize}

\begin{figure*}[t]
    \centering
    \includegraphics[width=\textwidth]{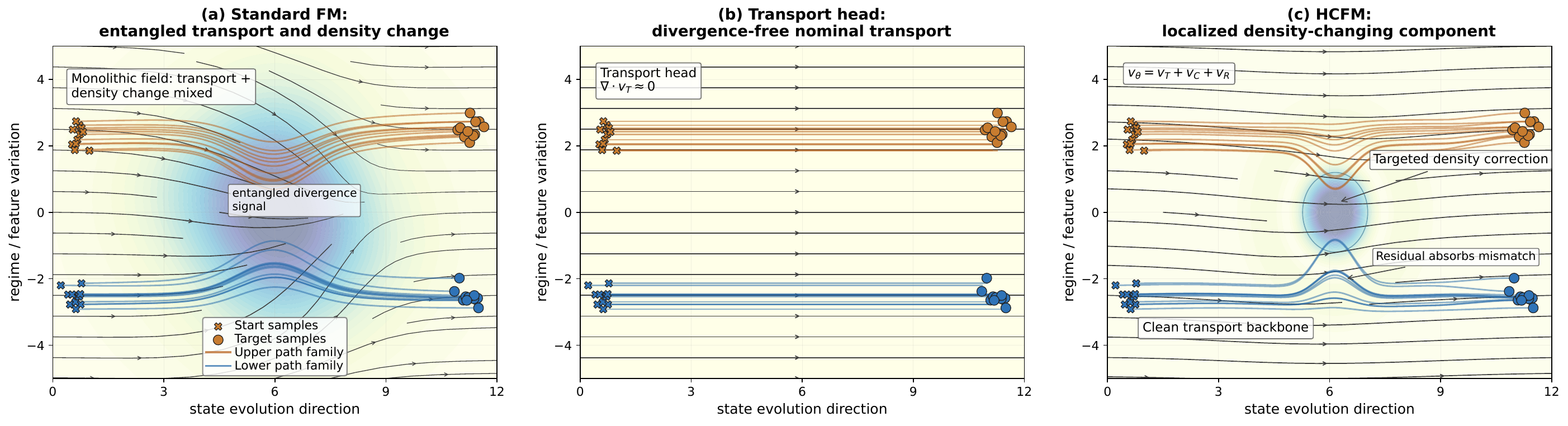}
    \caption{
    Motivating illustration of HCFM.
    (a) In standard Flow Matching, a monolithic predicted field entangles nominal transport with local density-changing behavior, producing an entangled divergence signal that bends otherwise clean transport paths.
    (b) The transport head in HCFM is constrained to be divergence-free, encouraging volume-preserving nominal movement without local expansion or contraction.
    (c) HCFM predicts a single final vector field through additive transport, compression, and residual heads, \(v_\theta = v_T + v_C + v_R\).
    This structured parameterization preserves a clean transport backbone while localizing density-changing behavior to the compression component and using the residual head to absorb residual mismatch.
    }
    \label{fig:hcfm_motivation}
\end{figure*}

\section{Related Work}

\subsection{Time-Series Anomaly Detection}

Time-series anomaly detection (TSAD) spans industrial monitoring, cyber-physical systems, cloud infrastructure, healthcare, spacecraft telemetry, and scientific sensing. Classical approaches include statistical process control, autoregressive modeling, change-point detection, distance-based subsequence analysis, and matrix-profile methods~\cite{yeh2016matrix, hundman2018detecting}; deep methods instead learn nominal temporal structure with recurrent, convolutional, graph-based, attention-based, and Transformer architectures~\cite{su2019robust, audibert2020usad, xu2022anomaly, tuli2022tranad}. Many deep detectors score by reconstruction or prediction error~\cite{su2019robust, geiger2020tadgan, audibert2020usad, tuli2022tranad}, yet expressive decoders can reconstruct anomalies well and limited models over-penalize uncommon but valid nominal patterns. Anomaly Transformer~\cite{xu2022anomaly} and TranAD~\cite{tuli2022tranad} improve long-range dependency modeling, but stronger sequence modeling does not by itself distinguish global deviation, local expansion or contraction, and residual mismatch, which HCFM separates by structuring the generative vector-field predictor.

\subsection{Generative Models and Flow Matching}

Generative detectors model nominal data and score deviations by likelihood, reconstruction, latent-space statistics, or transport. VAEs, GANs, normalizing flows, diffusion models, and score-based models have all been applied~\cite{kingma2014auto, rezende2014stochastic, dinh2017density, papamakarios2021normalizing, ho2020denoising, zamanzadeh2024deep}, each with a characteristic tradeoff: VAEs and autoencoders inherit decoder-driven scoring artifacts~\cite{saravanan2023tsi}, GANs suffer mode collapse and training instability~\cite{thanh2020catastrophic}, normalizing flows restrict architectures to maintain invertibility and tractable Jacobian determinants~\cite{kobyzev2020normalizing}, and diffusion models require iterative denoising or sampling. Flow Matching (FM) learns a vector field that transports a simple base distribution to the data distribution~\cite{lipman2023flow, tong2024improving}, avoiding invertibility constraints, iterative diffusion-style denoising, and sole reliance on reconstruction. Standard FM, however, predicts the field with a monolithic output that couples nominal transport with local density change, making expansion or contraction difficult to isolate as an anomaly signal. Recent time-series FM work includes TSFlow~\cite{kollovieh2025flow}, which adapts FM to probabilistic forecasting using Gaussian-process priors and optimal-transport couplings. Closer to our setting, SGFM~\cite{he2024sgfm} applies conditional FM to TSAD with state-space sequence modeling, but still predicts a monolithic field. HCFM instead structures the learned vector field itself into transport, compression, and residual components, so that volume-preserving transport and density-changing behavior provide distinct anomaly signals.

\subsection{Divergence Matching and Divergence-Free Flows}

Divergence measures the local expansion or contraction of a continuous-time vector field and governs density evolution along the flow. Flow Divergence Matching (FDM)~\cite{huang2026improving} trains the field to match both the target vector field and the divergence of the target probability path, addressing density information that directional matching alone may not capture. However, FDM retains a monolithic field in which transport and density change remain coupled. Divergence-free continuous flows~\cite{benaichouche3813499divergence} are volume-preserving normalizing flows: zero divergence gives incompressible dynamics with unit Jacobian determinant magnitude and simplifies quantities such as differential entropy, but constraining the entire flow prevents it from modeling arbitrary density changes. HCFM instead predicts a single FM field through additive heads and applies the constraint selectively: a divergence-free transport head preserves volume, a compression head captures density change and supplies component-specific divergence as anomaly evidence, and a residual head absorbs the remaining mismatch.

\subsection{Positioning of This Work}

HCFM differs from prior FM-based TSAD primarily in how the vector field is parameterized and diagnosed. It retains a single final field and standard FM supervision, but separates divergence-free transport, density-changing compression, and residual mismatch. Relative to standard FM, it exposes expansion and contraction through a dedicated compression component; relative to FDM, it localizes divergence scoring to that component; and unlike fully divergence-free flows, it preserves density-changing capacity in the overall model. Thus, HCFM provides a {\em component-structured parameterization of Flow Matching} for separating nominal transport from density-changing behavior and enabling targeted derivative-based diagnostics.

\begin{figure*}[t]
    \centering
    \includegraphics[width=\textwidth]{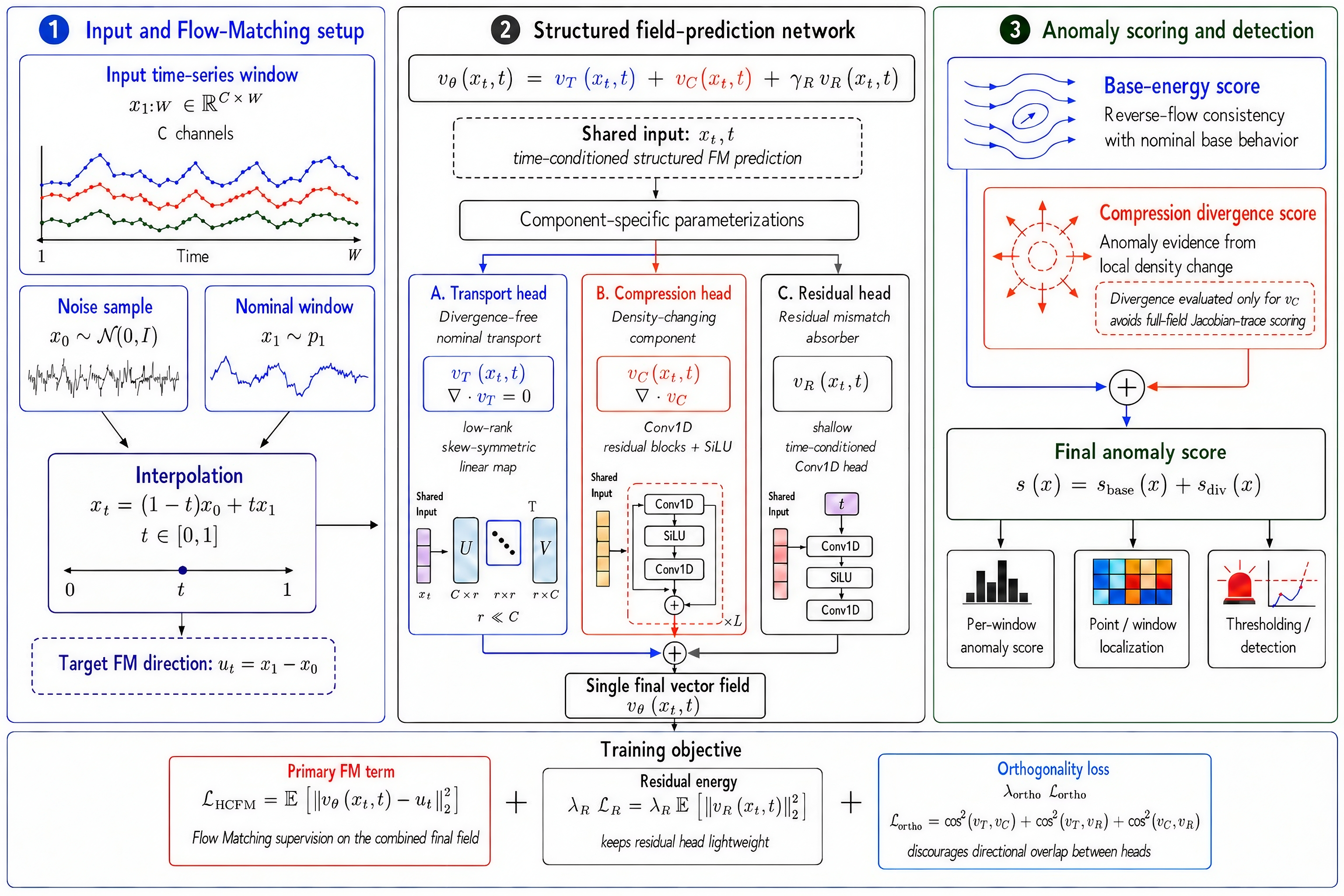}
    \caption{
    Architecture of Helmholtz Component Flow Matching (HCFM).
    HCFM trains a data-space Flow Matching model over interpolated time-series windows and replaces a monolithic vector-field predictor with additive transport, compression, and residual heads.
    The transport head captures divergence-free nominal movement, the compression head provides localized density-changing evidence through component-specific divergence, and the residual head absorbs remaining mismatch while being kept lightweight.
    Training supervises the combined final field and uses residual-energy and orthogonality regularization to shape the internal component allocation.
    At inference, HCFM combines base-energy and compression-divergence for anomaly detection and localization.
    }
    \label{fig:hcfm-architecture}
\end{figure*}

\section{Methodology}
\subsection{Problem Setup and Window Construction}
\label{sec:problem-setup}

We consider one-class time-series anomaly detection, where only nominal data are
available during training and the goal is to assign anomaly scores to future
test windows or timestamps. Let \(
Y=\{y_t\}_{t=1}^{T}\,,
y_t\in\mathbb{R}^{C},
\)
denote a multivariate time series with \(C\) channels and \(T\) timestamps. We
construct overlapping windows of length \(W\) as
\[
x_i
=
\big[y_i,y_{i+1},\ldots,y_{i+W-1}\big]
\in \mathbb{R}^{C\times W},
\]
where \(i\) denotes the window start index. In our experiments, we use stride
one unless stated otherwise.

Before windowing, each channel is standardized using statistics computed only
from the nominal training sequence:
\[
\tilde{y}_{t,c}
=
\mathrm{clip}
\left(
\frac{y_{t,c}-\mu_c}{\max(\sigma_c,\epsilon)},
-\kappa,\kappa
\right),
\]
where \(\mu_c\) and \(\sigma_c\) are the training mean and standard deviation of
channel \(c\), \(\epsilon>0\) is a scale floor, and \(\kappa\) is a clipping
threshold for numerical stability. The same transformation is then applied to
test data using the fixed training statistics, preventing information leakage
from the test sequence.

Let \(p_1\) denote the empirical distribution over nominal training windows.
HCFM learns a continuous-time vector field over the preprocessed window space
\(\mathbb{R}^{C\times W}\). During evaluation, each test window is assigned an anomaly score. Window-level scores are later projected
back to timestamps by aggregating the scores of all windows covering each
timestamp.

\subsection{Conditional Flow Matching for Time-Series Windows}
\label{sec:cfm-windows}

HCFM builds on Conditional Flow Matching (CFM), which trains a time-dependent
vector field to transport samples from a simple base distribution to the nominal
window distribution. For each nominal target window \(x_1\sim p_1\), we draw a
source sample
\(
x_0 \sim \mathcal{N}(0,I),
\)
with the same shape as \(x_1\). We then sample
\(
t\sim \mathcal{U}(0,1),
\)
and construct the intermediate state
\[
x_t
=
(1-t)x_0+t x_1.
\]
Under this linear path, the conditional target vector field is
\[
u_t
=
\frac{d x_t}{dt}
=
x_1-x_0.
\]

The goal is to learn a neural vector field
\[
v_\theta:
\mathbb{R}^{C\times W}\times[0,1]
\rightarrow
\mathbb{R}^{C\times W}
\]
that matches this target direction over nominal training windows. The standard
data-space Flow Matching objective is
\[
\mathcal{L}_{\mathrm{FM}}(\theta)
=
\mathbb{E}_{x_1,x_0,t}
\left[
\left\|
v_\theta(x_t,t)
-
u_t
\right\|_2^2
\right],
\]
where \(x_1\sim p_1\), \(x_0\sim\mathcal{N}(0,I)\), and
\(t\sim\mathcal{U}(0,1)\). This objective gives a
simulation-free regression target for learning nominal transport in the
preprocessed window space. Standard FM predicts \(v_\theta(x_t,t)\) with a
monolithic vector-field network. In contrast, HCFM keeps the same Flow Matching
supervision but replaces the monolithic output parameterization with a
structured additive field.

\subsection{Helmholtz Component Field Parameterization}
\label{sec:hcfm-parameterization}

Standard Flow Matching predicts \(v_\theta(x_t,t)\) with a monolithic output
head. This can entangle volume-preserving nominal movement, local expansion or
contraction of the probability path, and residual mismatch needed to fit the FM
target. HCFM instead parameterizes one final vector field through additive
transport, compression, and residual heads, as illustrated in
Fig.~\ref{fig:hcfm-architecture}:
\[
v_\theta(x_t,t)
=
v_T(x_t,t)
+
v_C(x_t,t)
+
\gamma_R v_R(x_t,t).
\]
Here, \(v_T\), \(v_C\), and \(v_R\) denote the transport, compression, and
residual heads, respectively. The scalar \(\gamma_R\) controls the contribution
of the residual component. HCFM still predicts a single final vector field: the
heads are internal components of one predictor, and Flow Matching supervision is
applied only to their combined output \(v_\theta\).

\paragraph{Transport head}
Let \(d=CW\), and let \(\bar{x}_t\in\mathbb{R}^{d}\) denote the flattened
intermediate window. The transport head is parameterized by a low-rank
skew-symmetric linear map,
\[
\bar{v}_T(\bar{x}_t,t)
=
A_\theta(t)\bar{x}_t,
\]
where
\[
A_\theta(t)
=
U_\theta(t)V_\theta(t)^\top
-
V_\theta(t)U_\theta(t)^\top\,,
U_\theta(t),V_\theta(t)\in\mathbb{R}^{d\times r}.
\]
The factors \(U_\theta(t)\) and \(V_\theta(t)\) are generated from a sinusoidal
embedding of \(t\), and \(r\ll d\) is the transport rank. Since
\(A_\theta(t)\) is skew-symmetric, Lemma~\ref{lem:transport-div-free} gives
\(\nabla_x\cdot v_T=0\). Consequently, the transport flow is locally
volume-preserving by Corollary~\ref{cor:transport-volume-preserving}, so this
head can model nominal movement without acting as a local source or sink.

\paragraph{Compression head}
The compression head is defined as the gradient of a learned scalar potential:
\[
v_C(x_t,t)
=
\nabla_{x_t}\phi_\theta(x_t,t),
\qquad
\phi_\theta:
\mathbb{R}^{C\times W}\times[0,1]\to\mathbb{R}.
\]
The potential \(\phi_\theta\) is implemented with a time-conditioned Conv1D
network using sinusoidal time embeddings, positional features, SiLU activations,
and residual Conv1D blocks. Its divergence is
\[
\nabla_x\cdot v_C(x,t)
=
\Delta_x\phi_\theta(x,t).
\]
By Lemma~\ref{lem:compression-density-change}, this divergence characterizes
instantaneous log-density change under the compression flow. We therefore use
\(\nabla_x\cdot v_C\) as a component-wise signal of local expansion or
contraction.

\paragraph{Residual head}
The residual head \(v_R(x_t,t)\) is an intentionally lightweight
time-conditioned Conv1D field. It consists of an input Conv1D projection, an
additive timestep embedding, a SiLU nonlinearity, and an output Conv1D
projection:
\[
v_R:
\mathbb{R}^{C\times W}\times[0,1]
\rightarrow
\mathbb{R}^{C\times W}.
\]
This head is called residual because it absorbs leftover mismatch after the
transport and compression components, not because it uses architectural skip
connections. It is downweighted by \(\gamma_R\) in the final field and
regularized during training so that it remains a low-capacity mismatch absorber.

\subsection{Training Objective}
\label{sec:hcfm-training-objective}

The primary HCFM loss supervises the combined vector field \(v_\theta\) using
the Flow Matching target:
\[
\mathcal{L}_{\mathrm{HCFM}}
=
\mathbb{E}_{x_1,x_0,t}
\left[
\left\|
v_\theta(x_t,t)
-
u_t
\right\|_2^2
\right].
\]
This loss is applied to the single combined field \(v_\theta\), not separately
to the individual component heads.

To keep the residual head from dominating the structured field, we penalize its
energy:
\[
\mathcal{L}_{R}
=
\mathbb{E}_{x_t,t}
\left[
\left\|
v_R(x_t,t)
\right\|_2^2
\right].
\]
This encourages the residual head to act as a lightweight mismatch absorber
rather than an unconstrained replacement for the transport and compression
heads.

We also include an orthogonality penalty that discourages directional overlap
among the component outputs:
\[
\mathcal{L}_{\mathrm{ortho}}
=
\mathbb{E}_{x_t,t}
\left[
\cos^2(v_T,v_C)
+
\cos^2(v_T,v_R)
+
\cos^2(v_C,v_R)
\right],
\]
where each cosine term is computed after flattening the corresponding component
fields. As shown in Lemma~\ref{lem:ortho-residual-duplication}, each term
equals the normalized squared energy of one head output projected onto the
direction of another. In particular, the terms involving \(v_R\) penalize
residual-head alignment with the transport and compression directions.

The full training objective is
\[
\mathcal{L}_{\mathrm{train}}
=
\mathcal{L}_{\mathrm{HCFM}}
+
\lambda_R \mathcal{L}_{R}
+
\lambda_{\mathrm{ortho}}\mathcal{L}_{\mathrm{ortho}}.
\]
Thus, the Flow Matching term supervises the final combined field, while the
auxiliary losses shape the internal allocation of responsibility among the
component heads.

\begin{lemma}[Divergence-free parameterization of the transport head]
Let \(v_T(x,t)=A(t)x\), where \(A(t)\in\mathbb{R}^{d\times d}\) is independent of \(x\) and skew-symmetric, i.e., \(A(t)^\top=-A(t)\). Then
\[
\nabla_x\cdot v_T(x,t)=0.
\]
\label{lem:transport-div-free}
\end{lemma}

\begin{proof}
Since \(A(t)\) is independent of \(x\), the Jacobian of \(v_T(x,t)=A(t)x\) with respect to \(x\) is \(A(t)\). Hence
\[
\nabla_x\cdot v_T(x,t)
=
\operatorname{tr}(A(t)).
\]
Because \(A(t)^\top=-A(t)\),
\[
\operatorname{tr}(A(t))
=
\operatorname{tr}(A(t)^\top)
=
\operatorname{tr}(-A(t))
=
-\operatorname{tr}(A(t)).
\]
Thus \(\operatorname{tr}(A(t))=0\), and therefore \(\nabla_x\cdot v_T(x,t)=0\).
\end{proof}

\begin{corollary}[Transport flow is volume-preserving]
Let \(\Phi_t\) denote the flow map generated by the transport head
\[
\frac{d x(t)}{dt}=v_T(x(t),t).
\]
If \(v_T(x,t)=A(t)x\) with \(A(t)^\top=-A(t)\), then the transport flow is locally volume-preserving:
\[
\left|\det \frac{\partial \Phi_t(x)}{\partial x}\right|=1.
\]
\label{cor:transport-volume-preserving}
\end{corollary}

\begin{proof}
Let \(J_t=\frac{\partial \Phi_t(x)}{\partial x}\) be the Jacobian of the transport flow map. The continuous change-of-variables relation gives
\[
\frac{d}{dt}\log |\det J_t|
=
\nabla_x\cdot v_T(x(t),t).
\]
By Lemma~\ref{lem:transport-div-free}, \(\nabla_x\cdot v_T(x(t),t)=0\). Hence
\[
\frac{d}{dt}\log |\det J_t|=0.
\]
Since \(J_0=I\), we have \(|\det J_0|=1\), and therefore \(\left|\det J_t\right|=1.\)
\end{proof}
This corollary formalizes the role of the transport head: it can model nominal movement along the learned flow without itself creating local expansion or contraction, leaving density-changing evidence to the compression head.

\begin{lemma}[Compression-head divergence characterizes density change]
Let \(v_C(x,t)\) be a smooth compression-head vector field, and let \(x(t)\) evolve according to
\[
\frac{dx(t)}{dt}=v_C(x(t),t).
\]
Let \(p_t(x)\) be a positive density transported by this compression flow. Then, along a trajectory \(x(t)\),
\[
\frac{d}{dt}\log p_t(x(t))
=
-\nabla_x\cdot v_C(x(t),t).
\]
Thus, positive compression divergence corresponds to local density decrease, while negative compression divergence corresponds to local density increase.
\label{lem:compression-density-change}
\end{lemma}

\begin{proof}
The transported density satisfies the continuity equation
\[
\frac{\partial p_t}{\partial t}
+
\nabla_x\cdot(p_t v_C)
=
0.
\]
Expanding the divergence term gives
\[
\frac{\partial p_t}{\partial t}
+
v_C\cdot\nabla_x p_t
+
p_t \nabla_x\cdot v_C
=
0.
\]
Along a trajectory satisfying \(dx(t)/dt=v_C(x(t),t)\),
\[
\frac{d}{dt}p_t(x(t))
=
\frac{\partial p_t}{\partial t}
+
v_C\cdot\nabla_x p_t.
\]
Therefore,
\[
\frac{d}{dt}p_t(x(t))
=
-p_t(x(t))\nabla_x\cdot v_C(x(t),t).
\]
Since \(p_t(x(t))>0\), dividing by \(p_t(x(t))\) yields
\[
\frac{d}{dt}\log p_t(x(t))
=
-\nabla_x\cdot v_C(x(t),t).
\]
\end{proof}

This lemma gives a component-wise interpretation of the compression head. Under the flow generated by \(v_C\), its divergence determines the instantaneous log-density change along trajectories. In the full HCFM field, the transport head is divergence-free by construction, while the residual head is lightweight and regularized. Thus, \(\nabla_x\cdot v_C\) serves as a component-specific signed signal of local expansion or contraction, rather than a claim that all density change in the full predicted field is caused solely by \(v_C\).

\begin{lemma}[Orthogonality loss penalizes residual-head directional duplication]
\label{lem:ortho-residual-duplication}
Let \(v_T,v_C,v_R\) be nonzero transport, compression, and residual head outputs at a fixed \((x,t)\). For any distinct pair \(v_i,v_j\in\{v_T,v_C,v_R\}\), let
\[
v_j^{\parallel i}
=
\frac{\langle v_j,v_i\rangle}{\|v_i\|_2^2}v_i
\]
denote the projection of \(v_j\) onto \(v_i\). Then
\[
\frac{\|v_j^{\parallel i}\|_2^2}{\|v_j\|_2^2}
=
\cos^2(v_j,v_i).
\]
Consequently, for
\[
\mathcal{L}_{\mathrm{ortho}}(x,t)
=
\cos^2(v_T,v_C)
+
\cos^2(v_T,v_R)
+
\cos^2(v_C,v_R),
\]
each term equals the normalized squared energy of one head output projected onto the direction of another. In particular, \(\cos^2(v_T,v_R)\) and \(\cos^2(v_C,v_R)\) penalize the fractions of residual-head energy aligned with the transport and compression directions, respectively. Moreover, if
\[
\mathcal{L}_{\mathrm{ortho}}(x,t)\le \epsilon
\qquad \text{for some } \epsilon\ge 0,
\]
then each pairwise aligned-energy fraction is at most \(\epsilon\).
\end{lemma}

\begin{proof}
By definition of projection,
\[
\|v_j^{\parallel i}\|_2^2
=
\left\|
\frac{\langle v_j,v_i\rangle}{\|v_i\|_2^2}v_i
\right\|_2^2
=
\frac{\langle v_j,v_i\rangle^2}{\|v_i\|_2^2}.
\]
Dividing by \(\|v_j\|_2^2\) gives
\[
\frac{\|v_j^{\parallel i}\|_2^2}{\|v_j\|_2^2}
=
\frac{\langle v_j,v_i\rangle^2}{\|v_j\|_2^2\|v_i\|_2^2}
=
\cos^2(v_j,v_i).
\]
Applying this identity to the pairs \((v_T,v_C)\), \((v_T,v_R)\), and \((v_C,v_R)\) gives the stated interpretation of \(\mathcal{L}_{\mathrm{ortho}}(x,t)\). Since all terms in \(\mathcal{L}_{\mathrm{ortho}}(x,t)\) are nonnegative, \(\mathcal{L}_{\mathrm{ortho}}(x,t)\le \epsilon\) implies that each individual alignment term is also at most \(\epsilon\).
\end{proof}

This result provides a pointwise geometric interpretation of the orthogonality penalty: it encourages the transport, compression, and residual heads to occupy distinct output directions at each \((x,t)\), penalizing the residual head from duplicating directions used by the transport or compression heads.

\subsection{Inference and Anomaly Scoring}
\label{sec:hcfm-scoring}

After training on nominal windows, HCFM scores each test window using two
signals: reverse-flow consistency with the Gaussian base distribution and
compression-head divergence along the learned trajectory. Given a test window
\(x_1\), we integrate the learned field backward from \(t=1\) to
\(t=0\), obtaining a base-space endpoint \(\hat{x}_0\) and trajectory
states \(\{(x_{t_k},t_k)\}_{k=1}^{K}\). Since nominal windows should map to
typical samples under the base distribution, we define the base-energy score as
\[
s_{\mathrm{base}}(x_1)
=
\frac{1}{2}\|\hat{x}_0\|_2^2.
\]

The second signal is the divergence of the compression head. At each trajectory
state, we estimate
\[
d_C(x_{t_k},t_k)
=
\nabla_x\cdot v_C(x_{t_k},t_k)
\]
using Hutchinson trace estimation. Unlike full-field divergence scoring, this
only evaluates the compression component, which is the head explicitly designed to capture local expansion and contraction. We aggregate the compression divergence
along the trajectory:
\[
s_{\mathrm{div}}(x_1)
=
\frac{1}{K}
\sum_{k=1}^{K}
d_C(x_{t_k},t_k).
\]
The main HCFM score combines base consistency with compression-divergence
evidence:
\[
s_{\mathrm{HCFM}}(x_1)
=
s_{\mathrm{base}}(x_1)
+
s_{\mathrm{div}}(x_1).
\]

For timestamp-level evaluation, window scores are projected back to individual
timestamps by averaging the scores of all windows covering each timestamp.

\subsection{Computational Complexity}
\label{sec:complexity}

HCFM reduces the cost of divergence-based scoring by avoiding Jacobian tracing over the full learned vector field. Let \(K\) denote the number of reverse-flow solver steps, \(P\) the number of Hutchinson probes, and \(\mathcal{C}_v\) the cost of one monolithic vector-field evaluation. Jacobian-free FM scoring only integrates the reverse ODE and computes the base energy, resulting in \(O(K\mathcal{C}_v)\) inference time. Although efficient, this score does not explicitly capture local density-change information.

Monolithic divergence-based scoring, including FDM-style scoring, additionally estimates \(\nabla\!\cdot v_\theta\) over the full vector field. With Hutchinson trace estimation, this gives \(O(K\mathcal{C}_v + KP\mathcal{C}^{J}_v)\) time, where \(\mathcal{C}^{J}_v\) is the cost of a Jacobian-vector or vector-Jacobian product through the full predictor. The derivative computation is therefore applied to the entire field, including components that may primarily model nominal transport rather than anomaly-relevant density change.

HCFM uses the same reverse-flow integration but restricts divergence scoring to the compression component. \(v_T\) is divergence-free by construction, \(v_C\) captures local compression and expansion, and \(v_R\) absorbs residual mismatch. The resulting inference cost is \(O(K(\mathcal{C}_T+\mathcal{C}_C+\mathcal{C}_R)+KP\mathcal{C}^{J}_C)\), where \(\mathcal{C}^{J}_C\) is the derivative cost of the compression head. Thus, HCFM preserves divergence-aware scoring while replacing full-field Jacobian tracing with component-specific derivative estimation.

Figure~\ref{fig:derivative-target} summarizes this allocation. HCFM is matched to the monolithic predictor in total parameter count, so the comparison does not rely on increased model capacity. However, divergence is evaluated only on the compression component. In this configuration, \(v_C\) contains \(37.6\%\) of the predictor parameters, so HCFM traces a \(0.376\times\) derivative target, corresponding to \(62.4\%\) fewer traced parameters than full-field divergence.

\begin{figure}[t]
    \centering
    \includegraphics[width=\columnwidth]{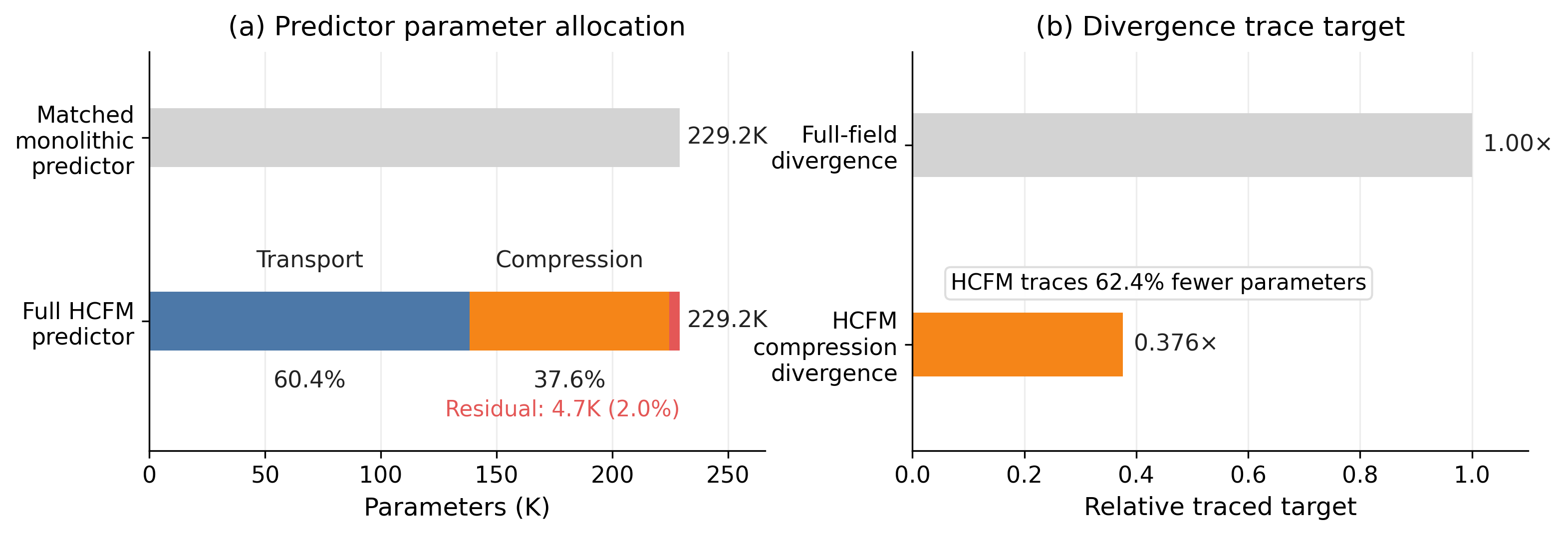}
    \caption{Parameter and divergence-trace target allocation for HCFM. (a) HCFM is matched to a monolithic predictor with the same total parameter budget. Most capacity is assigned to transport and compression, while the residual head remains lightweight. (b) HCFM differentiates only the compression component for divergence scoring, reducing the traced target to \(0.376\times\) of the full-field monolithic target in this configuration.}
    \label{fig:derivative-target}
\end{figure}

\begin{table}[t]
\caption{HCFM structured vector-field network and training and inference hyperparameters.}
\centering
\resizebox{\columnwidth}{!}{%
\begin{tabular}{cccc}
\toprule
Operation & Kernel / Stride & Units / Shape & Activation \\
\midrule
\textbf{Shared Input} \\
Input \(x_t\) & -- & \(B \times C \times W\) & -- \\
Timestep embedding & -- & \(d_{\mathrm{emb}}\) & SiLU \\
\midrule
\textbf{Transport Head} \\
Low-rank skew transport & -- & \(v_T(x_t,t)\in\mathbb{R}^{C\times W}\) & -- \\
\midrule
\textbf{Compression Head} \\
Input Conv1D & \((3,1)\) & hidden size & SiLU \\
ResidualBlock + GroupNorm (x3) & \((3,1)\) & hidden size & SiLU \\
Scalar-potential projection & -- & \(\phi_\theta(x_t,t)\in\mathbb{R}\) & -- \\
Gradient output & -- & \(v_C(x_t,t)=\nabla_{x_t}\phi_\theta(x_t,t)\) & -- \\
\midrule
\textbf{Residual Head} \\
Input Conv1D & \((3,1)\) & hidden size & -- \\
Add timestep embedding & -- & hidden size & SiLU \\
Conv1D output projection & \((3,1)\) & \(B \times C \times W\) & -- \\
Residual output & -- & \(v_R(x_t,t)\in\mathbb{R}^{C\times W}\) & -- \\
\midrule
\multicolumn{4}{l}{\textbf{Hyperparameters}} \\
Window size \(W=64\) & Hidden size \(128\) & \(d_{\mathrm{emb}}=32\) & Batch size \(256\) \\
Transport rank \(r=8\) & Iterations \(15{,}000\) & Learning rate \(1\mathrm{e}{-4}\) & Optimizer AdamW \\
Weight decay \(1\mathrm{e}{-4}\) & Residual scale \(\gamma_R=0.25\) & \(\lambda_R=5\mathrm{e}{-4}\) & \(\lambda_{\mathrm{ortho}}=10\) \\
ODE solver RK4  & ODE steps \(8\) & Hutchinson probes \(4\) \\
\bottomrule
\end{tabular}%
}
\label{tab:hcfm_arch}
\end{table}

\section{Experimental Setup}
\label{sec:experiment-setup}

We evaluate HCFM against representative time-series anomaly detection baselines
in both univariate and multivariate settings. For univariate anomaly detection,
we use the UCR Time Series Anomaly Archive. For multivariate anomaly detection,
we evaluate on GECCO, SWAN, MSL, SMAP, PSM, and SMD. These datasets span
diverse application domains, including spacecraft telemetry, server monitoring,
water-quality monitoring, and space-weather monitoring. Dataset statistics are
reported in Table~\ref{tab:dataset-stats}.

\begin{table}[t]
\centering
\caption{Statistics of the datasets used in our experiments. AR denotes anomaly ratio.}
\label{tab:dataset-stats}
\resizebox{\columnwidth}{!}{%
\begin{tabular}{lccccccc}
\toprule
\multirow{2}{*}{\textbf{Property}} &
\multicolumn{7}{c}{\textbf{Dataset}} \\
\cmidrule(lr){2-8}
& UCR & SMAP & MSL & GECCO & SMD & PSM & SWAN \\
\midrule
\# Sub-datasets     & 250       & 55      & 27     & 1      & 28      & 1      & 1      \\
\# Dimensions ($C$) & 1         & 25      & 55     & 9      & 38      & 25     & 38     \\
\# Train size       & 5,302,449 & 140,825 & 58,317 & 69,260 & 708,405 & 132,481 & 60,000 \\
\# Test size        & 14,051,317 & 444,035 & 73,729 & 69,261 & 708,420 & 87,841 & 60,000 \\
AR (\%) & 0.225 & 13.13 & 10.72 & 1.05 & 4.16 & 27.76 & 32.60 \\
\bottomrule
\end{tabular}%
}
\end{table}

\subsection{Datasets}
\label{sec:datasets}

\paragraph{\textbf{Univariate benchmark.}}
For the univariate setting, we use the UCR Time Series Anomaly Archive
\cite{keogh2021dataset}. The UCR benchmark contains 250 curated time series
from diverse real-world domains. Each dataset consists of a nominal training
segment and a test segment containing an annotated anomalous interval. The
benchmark includes point, contextual, and collective anomaly patterns, making it
a challenging testbed for one-class TSAD. Prior work has noted that many TSAD
benchmarks contain annotation errors, trivial anomalies, or unrealistic anomaly
densities \cite{wu2021current}; UCR was designed to mitigate many of these
issues and provides a more reliable evaluation environment.

\paragraph{\textbf{Multivariate benchmarks.}}
For the multivariate setting, we use six widely adopted real-world TSAD
benchmarks. SMAP and MSL contain spacecraft telemetry from NASA missions~\cite{hundman2018detecting}; PSM
and SMD contain server and infrastructure monitoring signals~\cite{abdulaal2021practical,su2019robust}; GECCO contains
water-quality measurements from a distribution monitoring system~\cite{moritz_2018_3884398}; and SWAN
contains space-weather monitoring signals~\cite{2020NatSD}. These datasets contain multichannel
time series with ground-truth anomaly intervals and cover diverse operational
conditions, anomaly densities, and inter-channel dependencies.

\subsection{Baselines and Variants}
\label{sec:baselines}

We compare HCFM against representative reconstruction-based,
Transformer-based, diffusion-based, contrastive, and flow-based baselines. For
reconstruction-based detection, we compare against established baselines such as
LSTM-AE, DENSE-AE, LSTM-DT, USAD~\cite{audibert2020usad},
OmniAnomaly~\cite{su2019robust}, DONUT~\cite{donut18},
TadGAN~\cite{geiger2020tadgan}, and TSI-GAN~\cite{saravanan2023tsi},
where applicable. We additionally include MERLIN~\cite{nakamura2020merlin},
a discord-based time-series anomaly detection baseline. For Transformer-based
detection, we include TranAD~\cite{tuli2022tranad} and Anomaly
Transformer~\cite{xu2022anomaly}. We also compare against recent diffusion- and
contrastive-learning baselines, including DiffAD~\cite{xiao2023imputation},
ImDiffusion~\cite{chen2023imdiffusion}, DCDetector~\cite{yang2023dcdetector},
and CARLA~\cite{darban2025carla}.

For flow-based baselines, we include vanilla data-space Flow Matching, Flow
Divergence Matching (FDM), and SGFM~\cite{he2024sgfm}. Vanilla FM uses the same
data-space Conditional Flow Matching objective as HCFM, but predicts the vector
field with a monolithic network. FDM incorporates divergence information into a
monolithic FM predictor. SGFM is a recent Flow-Matching-based TSAD method that
combines conditional Flow Matching with state-space sequence modeling. These
baselines isolate whether HCFM's gains arise from Flow Matching itself, from
divergence information alone, from sequence-modeling improvements in prior
FM-based TSAD, or from the proposed component-wise field parameterization.

For HCFM, the final model uses the transport, compression, and residual heads
described in Sec.~\ref{sec:hcfm-parameterization}. The main score family
combines reverse-flow base energy with compression-head divergence. We also
evaluate score variants based on transport energy, compression energy,
residual energy, centered compression divergence, full-divergence diagnostics,
and smoothed score transformations. These variants are used for ablation and
mechanism analysis.

\begin{table*}[t]
\centering
\caption{Aggregate UCR results over 250 datasets.}
\label{tab:ucr_aggregate_results}
\resizebox{\textwidth}{!}{%
\begin{tabular}{|l|c|c|c|c|c|c|c|}
\hline
\textbf{Model} &
\textbf{Best F1} $\uparrow$ &
\textbf{Precision} $\uparrow$ &
\textbf{VUS-PR} $\uparrow$ &
\textbf{AUROC} $\uparrow$ &
\textbf{AUPRC/AP} $\uparrow$ &
\textbf{FP@95R} $\downarrow$ &
\textbf{Hundman F1} $\uparrow$ \\
\hline
TadGAN~\cite{geiger2020tadgan} & 0.131 & 0.092 & 0.332 & 0.641 & 0.118 & 0.617 & 0.154 \\
\hline
DENSE-AE~\cite{alnegheimish2022sintel} & 0.159 & 0.136 & 0.352 & 0.662 & 0.147 & 0.584 & 0.181 \\
\hline
LSTM-DT~\cite{alnegheimish2022sintel} & 0.320 & 0.289 & 0.442 & 0.731 & 0.306 & 0.462 & 0.342 \\
\hline
DONUT~\cite{donut18} & 0.351 & 0.325 & 0.458 & 0.746 & 0.337 & 0.439 & 0.371 \\
\hline
LSTM-AE~\cite{alnegheimish2022sintel} & 0.355 & 0.301 & 0.472 & 0.754 & 0.351 & 0.427 & 0.384 \\
\hline
MERLIN~\cite{nakamura2020merlin} & 0.414 & 0.402 & 0.503 & 0.781 & 0.419 & 0.386 & 0.432 \\
\hline
TSI-GAN~\cite{saravanan2023tsi} & 0.468 & 0.445 & 0.529 & 0.804 & 0.462 & 0.349 & 0.489 \\
\hline
DiffAD~\cite{xiao2023imputation} & 0.479 & 0.452 & 0.532 & 0.811 & 0.474 & 0.337 & 0.501 \\
\hline
TranAD~\cite{tuli2022tranad} & 0.512 & 0.470 & 0.561 & 0.829 & 0.516 & 0.306 & 0.532 \\
\hline
ImDiffusion~\cite{chen2023imdiffusion} & 0.558 & 0.530 & 0.579 & 0.848 & 0.557 & 0.278 & 0.581 \\
\hline
Vanilla FM & 0.562 & 0.535 & 0.589 & 0.854 & 0.566 & 0.270 & 0.586 \\
\hline
SGFM~\cite{he2024sgfm} & 0.578 & 0.545 & 0.602 & 0.861 & 0.584 & 0.257 & 0.603 \\
\hline
FDM & 0.588 & 0.552 & 0.606 & 0.866 & 0.594 & 0.248 & 0.612 \\
\hline
CARLA~\cite{darban2025carla} & 0.597 & 0.560 & 0.609 & 0.872 & 0.603 & 0.239 & 0.621 \\
\hline
Anomaly Transformer~\cite{xu2022anomaly} & 0.637 & 0.590 & 0.631 & 0.887 & 0.642 & 0.216 & 0.663 \\
\hline
DCDetector~\cite{yang2023dcdetector} & 0.672 & 0.602 & 0.651 & 0.899 & 0.675 & 0.201 & 0.692 \\
\hline
\rowcolor{green!15}
HCFM & \textbf{0.742} & \textbf{0.705} & \textbf{0.694} & \textbf{0.922} & \textbf{0.757} & \textbf{0.171} & \textbf{0.776} \\
\hline
\end{tabular}
}
\end{table*}

\subsection{Evaluation Metrics}
\label{sec:metrics}

We evaluate anomaly detection using ranking-based, tolerance-aware, and thresholded metrics. Unless otherwise stated, metrics are computed at the timestamp level after projecting window scores back to timestamps using mean-overlap aggregation.

For ranking-based metrics, we report AUROC and AUPRC/AP, where AUPRC is computed as average precision (AP). AUPRC/AP is especially informative for imbalanced TSAD benchmarks because anomalous timestamps are typically rare. We also report VUS-PR, which extends precision--recall evaluation by incorporating temporal tolerance around anomaly intervals. We use a VUS-PR sliding tolerance window of \(W=64\) timestamps,
matching the time-series window length used during preprocessing.

For thresholded metrics, we report best F1, precision, FP@95R, and Hundman F1 where applicable. Best F1 is the maximum timestamp-level F1 over all score thresholds, and precision is reported at the same operating point. FP@95R reports the false-positive rate at the threshold required to achieve 95\% recall. For Hundman F1, predicted anomalous timestamps are grouped into contiguous intervals. A predicted interval is counted as a true positive if it overlaps any ground-truth anomaly interval; otherwise, it is counted as a false positive. Ground-truth intervals with no overlapping prediction are counted as false negatives, and F1 is computed from these interval-level counts.

The reported tables use these metrics as follows. For UCR, we report best F1, precision, VUS-PR, AUROC, AUPRC/AP, FP@95R, and Hundman F1. For multivariate per-dataset results, we report VUS-PR and best F1. For aggregate multivariate results, we report AUROC, AUPRC/AP, FP@95R, and Hundman F1. Aggregate multivariate results are computed by averaging each metric across the six datasets.

\subsection{Implementation Details}
\label{sec:implementation-details}

For HCFM, we use overlapping windows of length \(W=64\) with stride one across datasets; this fixed setting is not imposed on all baselines, which retain authors' recommended window settings when available. Each channel is standardized using training-only statistics, and the
same transformation is applied to the test sequence. HCFM is trained in data
space using the linear interpolation path described in
Sec.~\ref{sec:cfm-windows}. Table~\ref{tab:hcfm_arch} summarizes the HCFM architecture and training hyperparameters. All Flow Matching variants use matched preprocessing, training budget, and reverse-integration settings wherever applicable.

The transport head uses a low-rank skew-symmetric parameterization, the
compression head uses a scalar-potential Conv1D network with residual Conv1D
blocks and SiLU activations, and the residual head uses a lightweight
time-conditioned Conv1D field. The residual head is downweighted in the
combined vector field and regularized using the residual-energy penalty and
orthogonality loss described in Sec.~\ref{sec:hcfm-training-objective}.

During inference, test windows are reverse-integrated through the learned field
to obtain base-space endpoints and trajectory states. HCFM computes base-energy
and compression-divergence scores along this trajectory. Compression divergence
is estimated only for the compression head using Hutchinson trace estimation,
avoiding full-field Jacobian-trace scoring.

\subsection{Window-Size Sensitivity}
We additionally evaluate HCFM with \(W\in\{16,32,64,128\}\) across 15 UCR datasets while keeping the remaining protocol fixed. Relative to \(W=64\), mean Best F1 is approximately 30\% lower at \(W=16\) and 19\% lower at \(W=32\). \(W=128\) improves mean Best F1 by approximately 8\% but with larger variance across datasets. Thus, \(W=64\) serves as a stable common setting rather than a per-dataset optimized choice.

\section{Experimental Results}
\label{sec:experimental-results}

\begin{table*}[t]
\centering
\caption{Multivariate TSAD results.}
\label{tab:multivariate-results}
\resizebox{\textwidth}{!}{%
\begin{tabular}{lcccccccccccc}
\toprule
\multirow{2}{*}{Method} 
& \multicolumn{2}{c}{GECCO} 
& \multicolumn{2}{c}{SWAN} 
& \multicolumn{2}{c}{MSL} 
& \multicolumn{2}{c}{SMAP} 
& \multicolumn{2}{c}{PSM} 
& \multicolumn{2}{c}{SMD} \\
\cmidrule(lr){2-3}
\cmidrule(lr){4-5}
\cmidrule(lr){6-7}
\cmidrule(lr){8-9}
\cmidrule(lr){10-11}
\cmidrule(lr){12-13}
& VUS-PR $\uparrow$ & Best F1 $\uparrow$
& VUS-PR $\uparrow$ & Best F1 $\uparrow$
& VUS-PR $\uparrow$ & Best F1 $\uparrow$
& VUS-PR $\uparrow$ & Best F1 $\uparrow$
& VUS-PR $\uparrow$ & Best F1 $\uparrow$
& VUS-PR $\uparrow$ & Best F1 $\uparrow$ \\
\midrule
OmniAnomaly         & 0.612 & 0.661 & 0.746 & 0.790 & 0.628 & 0.678 & 0.578 & 0.635 & 0.703 & 0.742 & 0.677 & 0.725 \\
LSTM-AE             & 0.637 & 0.684 & 0.768 & 0.813 & 0.650 & 0.700 & 0.606 & 0.662 & 0.724 & 0.761 & 0.693 & 0.740 \\
DCDetector          & 0.681 & 0.728 & 0.805 & 0.846 & 0.698 & 0.745 & 0.648 & 0.701 & 0.756 & 0.789 & 0.733 & 0.780 \\
USAD                & 0.701 & 0.746 & 0.825 & 0.859 & 0.720 & 0.761 & 0.678 & 0.725 & 0.772 & 0.804 & 0.748 & 0.792 \\
TranAD              & 0.724 & 0.768 & 0.839 & 0.878 & 0.749 & 0.789 & 0.698 & 0.748 & 0.791 & 0.821 & 0.784 & 0.808 \\
CARLA               & 0.739 & 0.781 & 0.852 & 0.889 & 0.763 & 0.804 & 0.712 & 0.761 & 0.804 & 0.833 & 0.776 & 0.819 \\
DiffAD              & 0.766 & 0.806 & 0.862 & 0.899 & 0.782 & 0.821 & 0.738 & 0.782 & 0.826 & 0.852 & 0.789 & 0.831 \\
ImDiffusion         & 0.779 & 0.818 & 0.864 & 0.901 & 0.793 & 0.832 & 0.750 & 0.794 & 0.833 & 0.861 & 0.762 & 0.835 \\
Vanilla FM          & 0.784 & 0.825 & 0.872 & 0.910 & 0.802 & 0.839 & 0.756 & 0.801 & 0.842 & 0.868 & 0.771 & 0.824 \\
SGFM                & 0.791 & 0.833 & 0.878 & 0.913 & 0.811 & 0.848 & 0.764 & 0.807 & 0.850 & 0.875 & 0.799 & 0.842 \\
FDM                 & 0.801 & 0.841 & 0.884 & 0.918 & 0.819 & 0.857 & 0.772 & 0.812 & 0.856 & 0.879 & 0.806 & 0.848 \\
Anomaly Transformer & 0.811 & 0.849 & 0.889 & 0.922 & 0.824 & 0.861 & 0.780 & 0.820 & 0.862 & 0.884 & \textbf{0.823} & \textbf{0.860} \\
\rowcolor{green!12}
HCFM                & \textbf{0.848} & \textbf{0.881} & \textbf{0.912} & \textbf{0.941} & \textbf{0.846} & \textbf{0.879} & \textbf{0.815} & \textbf{0.852} & \textbf{0.888} & \textbf{0.909} & 0.812 & 0.849 \\
\bottomrule
\end{tabular}%
}
\end{table*}

\begin{table}[t]
\centering
\caption{Aggregate multivariate TSAD results across GECCO, SWAN, MSL, SMAP, PSM, and SMD.}
\label{tab:multivariate-agg-results}
\resizebox{\columnwidth}{!}{%
\begin{tabular}{lcccc}
\toprule
Method & AUROC $\uparrow$ & AUPRC/AP $\uparrow$ & FP@95R $\downarrow$ & Hundman F1 $\uparrow$ \\
\midrule
OmniAnomaly         & 0.781 & 0.521 & 0.774 & 0.705 \\
LSTM-AE             & 0.796 & 0.548 & 0.751 & 0.727 \\
DCDetector          & 0.823 & 0.592 & 0.704 & 0.782 \\
USAD                & 0.831 & 0.611 & 0.681 & 0.798 \\
TranAD              & 0.846 & 0.637 & 0.646 & 0.821 \\
CARLA               & 0.858 & 0.653 & 0.623 & 0.835 \\
DiffAD              & 0.871 & 0.684 & 0.592 & 0.854 \\
ImDiffusion         & 0.876 & 0.697 & 0.579 & 0.860 \\
Vanilla FM          & 0.882 & 0.712 & 0.562 & 0.870 \\
SGFM                & 0.889 & 0.724 & 0.548 & 0.878 \\
FDM                 & 0.895 & 0.734 & 0.531 & 0.884 \\
Anomaly Transformer & 0.901 & 0.744 & 0.514 & 0.891 \\
\rowcolor{green!12}
HCFM                & \textbf{0.914} & \textbf{0.768} & \textbf{0.486} & \textbf{0.905} \\
\bottomrule
\end{tabular}%
}
\end{table}

\subsection{Univariate Results}
\label{sec:ucr-main-results}

Table~\ref{tab:ucr_aggregate_results} shows that HCFM achieves the best result on every reported UCR metric, including best F1 of 0.742, AUROC of 0.922, AUPRC/AP of 0.757, and the lowest FP@95R of 0.171.

Compared with the best non-HCFM baseline, DCDetector, HCFM improves best F1 from 0.672 to 0.742, corresponding to a 10.4\% relative increase, and improves AUPRC/AP from 0.675 to 0.757, corresponding to a 12.1\% relative increase. HCFM also improves VUS-PR from 0.651 to 0.694 and Hundman F1 from 0.692 to 0.776, while reducing FP@95R from 0.201 to 0.171, a 14.9\% reduction in false positives at high recall.

Among Flow-Matching-based baselines, FDM is the best non-HCFM variant. Relative to FDM, HCFM improves best F1 from 0.588 to 0.742, corresponding to a 26.2\% relative increase, and improves AUPRC/AP from 0.594 to 0.757, corresponding to a 27.4\% relative increase. HCFM also reduces FP@95R from 0.248 to 0.171, a 31.0\% reduction. These gains reflect the advantage of component-specific divergence scoring: rather than estimating divergence over the full entangled vector field, HCFM concentrates the density-change signal in the compression head while retaining reverse-flow base-energy scoring.

The comparison with vanilla FM isolates the value of compression-divergence evidence and structured field prediction. Vanilla FM uses data-space Conditional Flow Matching with a monolithic field and base-energy scoring. Relative to vanilla FM, HCFM improves best F1 from 0.562 to 0.742, corresponding to a 32.0\% relative increase, improves AUPRC/AP from 0.566 to 0.757, corresponding to a 33.7\% relative increase, and reduces FP@95R from 0.270 to 0.171, a 36.7\% reduction. These improvements show that endpoint consistency alone is less informative than combining base energy with localized compression divergence.

\subsection{Main Results on Multivariate Benchmarks}
\label{sec:multivariate-main-results}

Table~\ref{tab:multivariate-results} reports per-dataset multivariate results across GECCO, SWAN, MSL, SMAP, PSM, and SMD. HCFM achieves the best VUS-PR and best F1 on five of the six datasets, including GECCO, SWAN, MSL, SMAP, and PSM, while remaining competitive on SMD. For example, on GECCO, HCFM improves VUS-PR from 0.801 to 0.848 and best F1 from 0.841 to 0.881 over FDM, corresponding to relative gains of 5.9\% and 4.8\%, respectively. On PSM, HCFM improves VUS-PR from 0.856 to 0.888 and best F1 from 0.879 to 0.909 over FDM, corresponding to gains of 3.7\% and 3.4\%.

The main exception is SMD, where Anomaly Transformer remains the best baseline with VUS-PR of 0.823 and best F1 of 0.860. HCFM remains competitive on this dataset, achieving VUS-PR of 0.812 and best F1 of 0.849, trailing Anomaly Transformer by 1.3\% in VUS-PR and 1.3\% in best F1. We conjecture that long-range temporal dependence may contribute to this small gap, but our analysis does not establish it as a causal explanation.

Table~\ref{tab:multivariate-agg-results} summarizes aggregate multivariate performance. HCFM achieves the best overall results, improving over the best non-HCFM aggregate baseline, Anomaly Transformer, from 0.901 to 0.914 in AUROC, from 0.744 to 0.768 in AUPRC/AP, and from 0.891 to 0.905 in Hundman F1. These correspond to relative improvements of 1.4\%, 3.2\%, and 1.6\%, respectively. HCFM also reduces FP@95R from 0.514 to 0.486, a 5.4\% reduction in false positives at high recall.

Compared with the best Flow-Matching baseline, FDM, HCFM improves AUROC from 0.895 to 0.914, AUPRC/AP from 0.734 to 0.768, and Hundman F1 from 0.884 to 0.905, corresponding to relative gains of 2.1\%, 4.6\%, and 2.4\%, respectively. HCFM also reduces FP@95R from 0.531 to 0.486, an 8.5\% reduction. Overall, these results support the central claim of HCFM: structuring the Flow Matching vector field into divergence-free transport, density-changing compression, and lightweight residual components yields more informative anomaly scores than either Jacobian-free base-energy FM or monolithic divergence-based scoring.
\subsection{Training and Inference Time Analysis} 
\label{sec:inference-time}

We further compare the computational cost of HCFM against representative TSAD
baselines. Training time is reported as wall-clock minutes, while inference latency is
reported as average milliseconds per window. The training and inference measurements in Table~\ref{tab:inference-time-baselines} are obtained on UCR dataset \texttt{001\_UCR\_Anomaly\_DISTORTED1sddb40} with \(C=1\), \(W=64\), stride one, and 44,732 evaluated test windows. Timing results are measured on a single NVIDIA RTX 4070 GPU after a warm-up pass, excluding data loading. For flow-based methods, inference
time includes reverse ODE integration and score construction. For
divergence-based methods, it also includes the derivative evaluations required
for divergence scoring. Relative inference time is normalized to HCFM.

Table~\ref{tab:inference-time-baselines} shows that HCFM provides a favorable
accuracy--efficiency tradeoff. HCFM is slower than lightweight reconstruction
or Jacobian-free flow baselines such as LSTM-AE and vanilla FM, since it
additionally evaluates compression-head divergence. However, it remains
substantially faster than FDM and several recent deep TSAD baselines. Compared
with FDM, HCFM reduces training time from \(31.5\) to \(22.8\) minutes and
inference latency from \(6.03\) to \(3.92\) ms/window. HCFM is also faster at
inference than TranAD, SGFM, Anomaly Transformer, CARLA, DCDetector, and
ImDiffusion.

These results indicate that HCFM is not merely an accuracy-oriented model with
prohibitive cost. Instead, it improves detection performance over vanilla FM
and FDM while keeping both training and inference time moderate. The main
efficiency benefit comes from localizing derivative-based scoring to the
compression component \(v_C\), rather than estimating divergence over the full
entangled vector field.

\begin{table}[t]
\centering
\caption{
Training and inference-time comparison against representative TSAD baselines.
Training time is reported as wall-clock minutes, while latency is reported as
average milliseconds per window. Relative inference time is normalized to HCFM.
}
\label{tab:inference-time-baselines}
\resizebox{\columnwidth}{!}{%
\begin{tabular}{lccc}
\toprule
Method &
Training time (min) &
Latency (ms/window) &
Rel. time \\
\midrule
LSTM-AE &
8.6 &
1.74 &
\(0.44\times\) \\

Vanilla FM &
14.2 &
2.05 &
\(0.52\times\) \\

\rowcolor{green!12}
HCFM &
22.8 &
3.92 &
\(\mathbf{1.00\times}\) \\

FDM &
31.5 &
6.03 &
\(1.54\times\) \\

TranAD &
24.7 &
7.64 &
\(1.95\times\) \\

SGFM &
36.9 &
8.13 &
\(2.07\times\) \\

Anomaly Transformer &
42.4 &
10.85 &
\(2.77\times\) \\

CARLA &
29.6 &
11.02 &
\(2.81\times\) \\

DCDetector &
48.2 &
18.43 &
\(4.70\times\) \\

ImDiffusion &
61.7 &
24.80 &
\(6.33\times\) \\
\bottomrule
\end{tabular}%
}
\end{table}

\subsection{Dimensional Scalability}
To further assess dimensional scalability, we varied the channel dimension from
\(C=1\) to \(64\) while fixing \(W=64\) and the number of evaluated windows.
Table~\ref{tab:dimensional-scalability} summarizes the resulting training and
inference costs. Increasing \(C\) by \(64\times\) increased training time by
\(2.15\times\), while inference latency per window increased by only
\(1.04\times\). The controlled setting was constructed by replicating the single
channel of \texttt{001\_UCR\_Anomaly\_DISTORTED1sddb40} along the channel axis
so that dimensionality could vary while keeping the temporal data, window
length, and number of windows fixed.

\begin{table}[h]
\centering
\caption{HCFM dimensional scalability with \(W=64\). Values in parentheses are relative
to \(C=1\).}
\label{tab:dimensional-scalability}
\resizebox{\columnwidth}{!}{%
\begin{tabular}{rcc}
\toprule
Channels \(C\) &
Training time (min) &
Latency (ms/window) \\
\midrule
1  & 22.80 (\(1.00\times\)) & 3.92 (\(1.00\times\))  \\
8  & 23.92 (\(1.05\times\)) & 3.98 (\(1.01\times\))  \\
16 & 26.52 (\(1.16\times\)) & 4.03 (\(1.03\times\))  \\
32 & 34.06 (\(1.49\times\)) & 3.98 (\(1.01\times\))  \\
64 & 49.04 (\(2.15\times\)) & 4.06 (\(1.04\times\))  \\
\bottomrule
\end{tabular}%
}
\end{table}

\section{Ablation Studies}
\label{sec:ablations}

We ablate three design choices in HCFM: (i) whether compression-head
divergence provides a more informative density-change signal than
monolithic full-field divergence, (ii) whether the
transport--compression--residual decomposition is responsible for the
gains and (iii) whether residual scaling,
residual-energy regularization, and orthogonality improve component allocation. Unless otherwise stated, all variants use the same preprocessing, window size, training budget, reverse ODE solver, and evaluation protocol as full HCFM, so differences in performance can be attributed to the ablated design
choices rather than changes in the experimental setup.

\subsection{Divergence Source and Scoring Efficiency}
\label{sec:ablation-divergence-efficiency}
Table~\ref{tab:ablation-divergence-efficiency} compares base-energy scoring without divergence, monolithic full-field divergence, and HCFM compression-head divergence. Adding full-field divergence to vanilla FM gives only modest gains, improving AUPRC from 0.566 to 0.578, while increasing relative inference time from \(1.00\times\) to \(2.71\times\). FDM improves further to 0.594 AUPRC, but remains costly because it estimates divergence over the full field \(v_\theta\). HCFM obtains stronger performance with lower derivative overhead by restricting divergence estimation to the compression head \(v_C\). 

HCFM base-only scoring already improves over vanilla FM, showing that the structured field benefits endpoint consistency. Compression divergence alone is stronger than base-only scoring, and the combined score performs best, reaching 0.757 AUPRC, 0.694 VUS-PR, 0.742 best F1, and the lowest FP@95R of 0.171 with only \(1.91\times\) relative inference time. Thus, the comparison separates the effects of Flow Matching, divergence-based scoring, and {\em structured localization of divergence}. The results show that simply adding divergence to a monolithic predictor is insufficient; localizing divergence to the compression component provides the stronger accuracy–efficiency tradeoff.

\begin{figure*}[t]
    \centering
    \includegraphics[width=\textwidth]{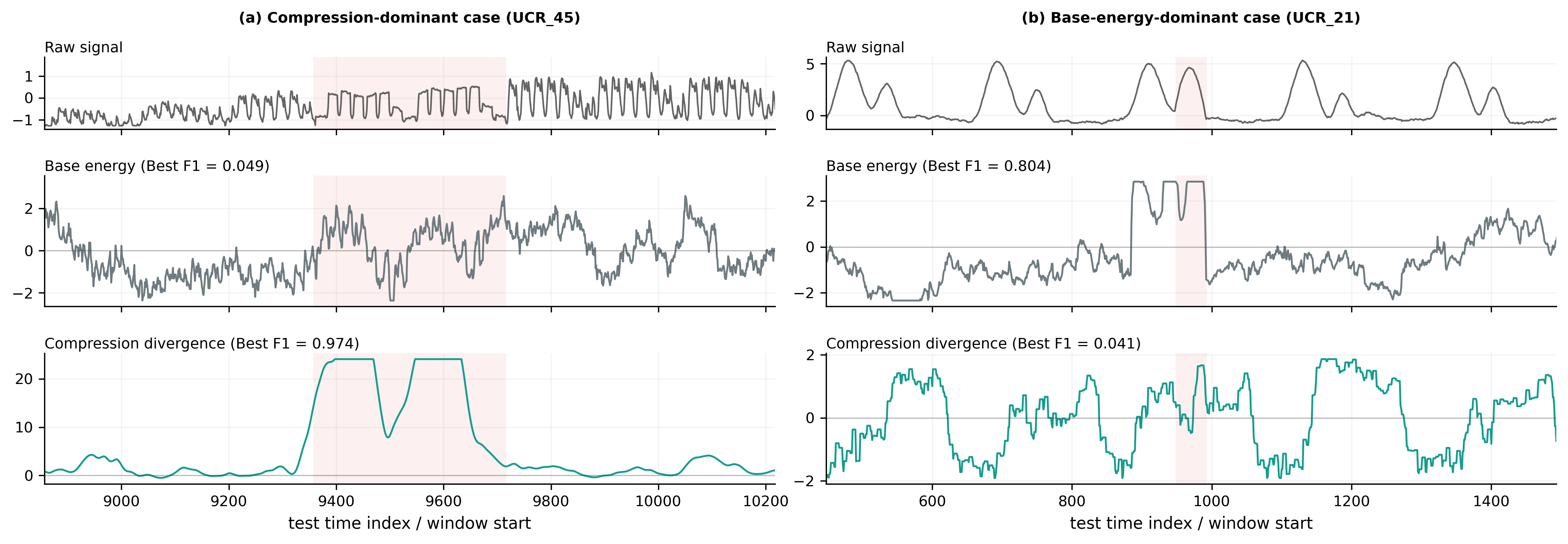}
    \caption{
    Qualitative examples illustrating complementary anomaly evidence in HCFM.
    The shaded region denotes the ground-truth anomaly interval. Each column shows
    the raw test signal, reverse-flow base energy, and compression-divergence score.
    (a) On UCR\_45, base energy provides weak anomaly separation, whereas
    compression divergence sharply localizes the anomalous interval.
    (b) On UCR\_21, the anomaly is more clearly captured by base energy than by
    compression divergence. Together, these examples motivate combining endpoint
    mismatch evidence with localized compression-divergence evidence.
    }
    \label{fig:qualitative-components}
\end{figure*}

\begin{table*}[t]
\centering
\caption{
Ablation of divergence source and scoring efficiency.
Runtime and memory are reported relative to Jacobian-free vanilla FM
base-energy scoring.
}
\label{tab:ablation-divergence-efficiency}
\resizebox{\textwidth}{!}{%
\begin{tabular}{lccccccc}
\toprule
Scoring method &
Divergence target &
AUPRC &
FP@95R &
VUS-PR &
Best F1 &
Rel. time &
Rel. memory \\
\midrule
Vanilla FM: base energy &
None &
0.566 & 0.270 & 0.589 & 0.562 & \(1.00\times\) & \(1.00\times\) \\

Vanilla FM: base \(+\) full divergence &
Full field \(v_\theta\) &
0.578 & 0.260 & 0.598 & 0.574 & \(2.71\times\) & \(2.18\times\) \\

FDM: base \(+\) full divergence &
Full field \(v_\theta\) &
0.594 & 0.248 & 0.606 & 0.588 & \(2.94\times\) & \(2.31\times\) \\

HCFM: base only &
None &
0.682 & 0.206 & 0.651 & 0.671 & \(1.18\times\) & \(1.08\times\) \\

HCFM: compression divergence only &
Compression head \(v_C\) &
0.735 & 0.180 & 0.684 & 0.724 & \(1.86\times\) & \(1.42\times\) \\

\rowcolor{green!12}
HCFM: base \(+\) compression divergence &
Compression head \(v_C\) &
\textbf{0.757} & \textbf{0.171} & \textbf{0.694} & \textbf{0.742} & \(1.91\times\) & \(1.46\times\) \\
\bottomrule
\end{tabular}%
}
\end{table*}

\subsection{Complementarity of Base Energy and Compression Divergence}
\label{sec:ablation-score-complementarity}

Figure~\ref{fig:qualitative-components} illustrates why combining base energy
and compression divergence is useful. On UCR\_45, base energy provides weak
separation while compression divergence sharply localizes the anomaly. On
UCR\_21, base energy captures the anomalous interval more clearly than
compression divergence. Thus, the two scores provide complementary anomaly evidence:
base energy measures endpoint consistency with the nominal Gaussian base, while
compression divergence captures localized density-changing deformation along
the trajectory. Combining the two therefore allows HCFM to exploit both global deviation from the nominal base distribution and local deformation of the learned flow.

\subsection{Structured Decomposition Versus Monolithic Fields}
\label{sec:ablation-decomposition}

Table~\ref{tab:ablation-decomposition} evaluates whether the structured HCFM
parameterization is responsible for the gains. Monolithic FM and monolithic FM
with full-field divergence remain well below full HCFM, indicating that the
improvement is not explained by Flow Matching alone or by simply adding
divergence to an unstructured predictor.

Removing the compression head causes the largest drop, confirming that
compression divergence is the main density-changing anomaly signal. Removing
the transport head also hurts performance, suggesting that compression and
residual components should not be forced to model both nominal movement and
density change. The no-residual variant remains competitive but trails full
HCFM, showing that a lightweight residual head helps absorb remaining mismatch
without dominating the field.

\begin{table*}[t]
\centering
\caption{
Ablation of the HCFM structured field decomposition.
All variants use matched preprocessing, training budget, and reverse-integration
settings.
}
\label{tab:ablation-decomposition}
\resizebox{\textwidth}{!}{%
\begin{tabular}{lccccccc}
\toprule
Variant &
Field parameterization &
Score used &
AUPRC &
FP@95R &
VUS-PR &
Best F1 &
Hundman F1 \\
\midrule
Vanilla FM &
\(v_\theta\) &
Base energy &
0.566 & 0.270 & 0.589 & 0.562 & 0.586 \\

Vanilla FM \(+\) divergence &
\(v_\theta\) &
Base \(+\nabla\!\cdot v_\theta\) &
0.578 & 0.260 & 0.598 & 0.574 & 0.598 \\

No transport head &
\(v_C+\gamma_R v_R\) &
Base \(+\nabla\!\cdot v_C\) &
0.724 & 0.185 & 0.676 & 0.713 & 0.743 \\

No compression head &
\(v_T+\gamma_R v_R\) &
Base energy &
0.647 & 0.226 & 0.632 & 0.636 & 0.665 \\

No residual head &
\(v_T+v_C\) &
Base \(+\nabla\!\cdot v_C\) &
0.743 & 0.178 & 0.686 & 0.728 & 0.759 \\

\rowcolor{green!12}
Full HCFM &
\(v_T+v_C+\gamma_R v_R\) &
Base \(+\nabla\!\cdot v_C\) &
\textbf{0.757} & \textbf{0.171} & \textbf{0.694} & \textbf{0.742} & \textbf{0.776} \\
\bottomrule
\end{tabular}%
}
\end{table*}

\subsection{Component Allocation and Residual Control}
\label{sec:ablation-component-allocation}

We finally study whether residual scaling, residual-energy regularization, and
orthogonality help maintain a clean component allocation. The residual head
should absorb leftover mismatch without becoming a second density-changing
component. We quantify this using the residual-to-compression divergence ratio
\[
\rho_{\mathrm{div}}
=
\frac{
\mathbb{E}_{x,t}\left[|\nabla\!\cdot v_R(x,t)|\right]
}{
\mathbb{E}_{x,t}\left[|\nabla\!\cdot v_C(x,t)|\right]
}.
\]
Large \(\rho_{\mathrm{div}}\) indicates that the residual head is absorbing
density-changing behavior that should remain localized in the compression head.

Table~\ref{tab:ablation-residual-control} shows that both \(\gamma_R\) and
\(\lambda_R\) matter. A large unregularized residual
\((\gamma_R=1.0,\lambda_R=0)\) performs poorly and gives
\(\rho_{\mathrm{div}}=1.721\), indicating residual dominance. Reducing the
residual scale and adding residual-energy regularization improves both
performance and component allocation. The default setting
\((\gamma_R=0.25,\lambda_R=5{\times}10^{-4})\) achieves the best detection
results, while a stronger residual penalty further lowers \(\rho_{\mathrm{div}}\)
but slightly hurts performance.

\begin{table}[h]
\centering
\caption{
Residual-scale and residual-regularization ablation for HCFM.
All variants use the same transport and compression heads with
\(\lambda_{\mathrm{ortho}}=10\). Lower \(\rho_{\mathrm{div}}\) indicates less residual-head
absorption of density-changing behavior.
}
\label{tab:ablation-residual-control}
\resizebox{\columnwidth}{!}{%
\begin{tabular}{cccccc}
\toprule
\(\gamma_R\) &
\(\lambda_R\) &
AUPRC &
FP@95R \(\downarrow\) &
Best F1 &
\(\rho_{\mathrm{div}}\downarrow\) \\
\midrule
\(1.00\) &
\(0\) &
0.686 & 0.232 & 0.668 & 1.721 \\

\(1.00\) &
\(5{\times}10^{-4}\) &
0.715 & 0.207 & 0.697 & 1.084 \\

\(0.50\) &
\(0\) &
0.723 & 0.199 & 0.705 & 0.923 \\

\(0.50\) &
\(5{\times}10^{-4}\) &
0.741 & 0.181 & 0.727 & 0.514 \\

\(0.25\) &
\(0\) &
0.731 & 0.193 & 0.713 & 0.682 \\

\rowcolor{green!12}
\(0.25\) &
\(5{\times}10^{-4}\) &
\textbf{0.757} & \textbf{0.171} & \textbf{0.742} & 0.287 \\

\(0.25\) &
\(1{\times}10^{-3}\) &
0.750 & 0.174 & 0.735 & \textbf{0.193} \\
\bottomrule
\end{tabular}%
}
\end{table}

Table~\ref{tab:hcfm_ortho_ablation} shows that moderate orthogonality is also
important. Increasing \(\lambda_{\mathrm{ortho}}\) improves performance up to
\(\lambda_{\mathrm{ortho}}=10\), the setting used in the main model. A larger
value of \(20\) slightly degrades detection, suggesting that excessive
separation among heads can become over-constraining.

\begin{table}[h]
\centering
\caption{
Orthogonality-weight ablation for HCFM.
}
\label{tab:hcfm_ortho_ablation}
\resizebox{\columnwidth}{!}{%
\begin{tabular}{lccccc}
\toprule
\(\lambda_{\mathrm{ortho}}\) &
Best F1 &
AUPRC &
AUROC &
VUS-PR &
FP@95R \(\downarrow\) \\
\midrule
\(0\)  & 0.704 & 0.721 & 0.912 & 0.671 & 0.194 \\
\(1\)  & 0.724 & 0.739 & 0.917 & 0.684 & 0.183 \\
\(2\)  & 0.732 & 0.748 & 0.920 & 0.689 & 0.177 \\
\(10\) & \textbf{0.742} & \textbf{0.757} & \textbf{0.922} & \textbf{0.694} & \textbf{0.171} \\
\(20\) & 0.736 & 0.751 & 0.919 & 0.690 & 0.174 \\
\bottomrule
\end{tabular}%
}
\end{table}

\section{Conclusion and Future Work}
\label{sec:conclusion}

We introduced Helmholtz Component Flow Matching (HCFM), an unsupervised TSAD framework that addresses the entanglement of nominal transport and local density-changing behavior in monolithic FM predictors. HCFM parameterizes one final field through divergence-free transport, density-changing compression, and lightweight residual components, and combines reverse-flow base energy with compression-head divergence for complementary anomaly evidence.

Experiments on 250 UCR datasets and six multivariate benchmarks show consistent gains over vanilla FM, FDM, and strong TSAD baselines, with improved detection quality and lower divergence-scoring cost. Ablations further show that the gains arise from the structured field parameterization rather than simply adding divergence: removing the compression head causes the largest degradation, while the transport and residual components also contribute to component allocation.

More broadly, HCFM suggests {\em component-structured parameterization as a design principle for Flow Matching} when different vector-field roles require different constraints or diagnostics. 

Future work will explore incorporating attention mechanisms into the component-structured predictor, especially for multivariate time series with long-range temporal dependencies and cross-channel interactions. This could allow HCFM to retain its transport-compression decomposition while improving global context modeling.

\section*{AI Usage Disclosure}

ChatGPT was used in polishing some of the writing and drafting/refining the architecture diagram. All conceptual ideas, formulations, experiments, results, and conclusions are entirely the authors' original work. The authors retain full responsibility for the final content of this work.

\bibliographystyle{IEEEtran}
\bibliography{icde}

\end{document}